\documentclass[letter, 10pt, conference]{ieeeconf}      
\IEEEoverridecommandlockouts                              

\usepackage{amsmath,amssymb}

\DeclareMathOperator{\EX}{\mathbb{E}}

\usepackage{rotating}
\usepackage{chngcntr}
\usepackage{apptools}
\AtAppendix{\counterwithin{thm}{section}}
\usepackage{graphicx}
\usepackage{graphics}
\usepackage{amssymb}
\usepackage{amsmath}
\usepackage{color}
\usepackage{xspace}
\usepackage{bbm}
\usepackage{comment}
\usepackage{algorithm}
\usepackage{algpseudocode}
\usepackage{subfig}
\usepackage{psfrag}
\usepackage{cite}
\usepackage{tikz}
\usetikzlibrary{shapes,arrows}
\usetikzlibrary{positioning}

\tikzstyle{block}=[draw opacity=0.7,line width=1.4cm]

\definecolor{CranJ}{cmyk}{0,0.69,0.54,0.04} 
\definecolor{PinkJ}{cmyk}{0,0.71,0.43,0.12} 
\definecolor{Cran}{cmyk}{0,0.73,0.41,0.29} 
\definecolor{VRed}{cmyk}{0,0.75,0.25,0.2} 
\definecolor{ORed}{cmyk}{0,0.75,0.75,0} 
\definecolor{CBlue}{cmyk}{1,0.25,0,0} 

\title{\LARGE \bf 
FedScalar: Federated Learning with Scalar Communication for Bandwidth-Constrained Networks
}

\author{Mohammadreza Rostami and Solmaz S. Kia, \emph{Senior Member, IEEE} %
  \thanks{The authors are with the Department of Mechanical and Aerospace Engineering, University of California Irvine, Irvine, CA 92697,  
    {\tt\small \{mrostam2,solmaz\}@uci.edu}. }%
}

\newcommand{\real}{{\mathbb{R}}} \newcommand{\reals}{{\mathbb{R}}}
 
\newcommand{\realpositive}{{\mathbb{R}}_{>0}}
\newcommand{\realnonnegative}{{\mathbb{R}}_{\ge 0}}

\renewcommand{\baselinestretch}{.985}

\newcommand{\vect}[1]{\boldsymbol{\mathbf{#1}}}

 \newcommand{\boxend}{\hfill \ensuremath{\Box}}

\allowdisplaybreaks

\newtheorem{thm}{Theorem}[section]
\newtheorem{prop}{Proposition}[section]

\newtheorem{lem}{Lemma}[section]
\newtheorem{defn}{Definition}

\newtheorem{assump}{Assumption}

\newcommand{\oprocendsymbol}{\hbox{$\bullet$}}
\newcommand{\oprocend}{\relax\ifmmode\else\unskip\hfill\fi\oprocendsymbol}

\makeatletter
\renewcommand*{\@opargbegintheorem}[3]{\trivlist
      \item[\hskip \labelsep{ #1\ #2}] (#3):\ \itshape}
\makeatother

\begin{document}\fontsize{10}{11.1}\rm

\maketitle
\thispagestyle{empty}
\pagestyle{empty}

\begin{abstract}
In bandwidth-constrained federated learning~(FL) settings, 
the repeated upload of high-dimensional model updates from 
agents to a central server constitutes the primary bottleneck, 
often rendering standard FL infeasible within practical 
communication budgets. We propose \emph{FedScalar}, a 
communication-efficient FL algorithm in which each agent 
uploads only two scalar values per round, regardless of 
the model dimension~$d$. Each agent encodes its local 
update difference as an inner product with a locally 
generated random vector and transmits the resulting scalar 
together with the generating seed, enabling the server to 
reconstruct an unbiased gradient estimate without any 
high-dimensional transmission. We prove that \emph{FedScalar} 
achieves a convergence rate of $O(d/\sqrt{K})$ to a 
stationary point for smooth non-convex loss functions, and 
show that adopting a Rademacher distribution for the random 
vector reduces the aggregation variance compared to the 
Gaussian case. Numerical simulations confirm that the 
dimension-free upload cost translates into significant 
improvements in wall-clock time and energy efficiency over 
\emph{FedAvg} and \emph{QSGD} in bandwidth-constrained 
settings.

\end{abstract}



\section{Introduction}

In recent years, Federated Learning (FL)~\cite{HBM-EM-DR-SH-BAA:17} has emerged
as a dominant paradigm for decentralized machine learning, enabling model training
across distributed agents without necessitating raw data aggregation. This
privacy-preserving framework has seen widespread adoption in high-stakes domains,
enabling scalable and privacy-preserving model training across
heterogeneous devices~\cite{kairouz2021advances, li2020federated}.
It has also recently shown its
relevance and utility in multi-agent data-driven control~\cite{dobbe2020toward, wang2023model} and reinforcement
learning~\cite{zhuo2019federated}. In a canonical FL setting, a group of $N$ agents
communicates with a central server to collaboratively learn the weights
$\mathbf{x} \in \mathbb{R}^d$ of a shared model by jointly solving
\begin{equation}
\label{eq::sum}
    \min_{\vect{x}\in\real^d}f(\vect{x}) := \frac{1}{N}\sum\nolimits_{n=1}^{N}f_n(\vect{x})
\end{equation}
where
$f_n : \mathbb{R}^d \rightarrow \mathbb{R}$ is the local loss function associated
with the private dataset of agent $n$. This collaborative optimization is carried
out without the agents ever explicitly sharing their local data with one another.
To this end, a central server orchestrates the learning process by broadcasting
global model parameters to a subset of clients, each of which performs local
stochastic gradient descent~(SGD) updates and transmits the resulting parameters
back to the server for aggregation, typically via the Federated Averaging~(FedAvg)
algorithm~\cite{HBM-EM-DR-SH-BAA:17}.

\begin{figure}[t]
\centering
    \includegraphics[scale=0.17]{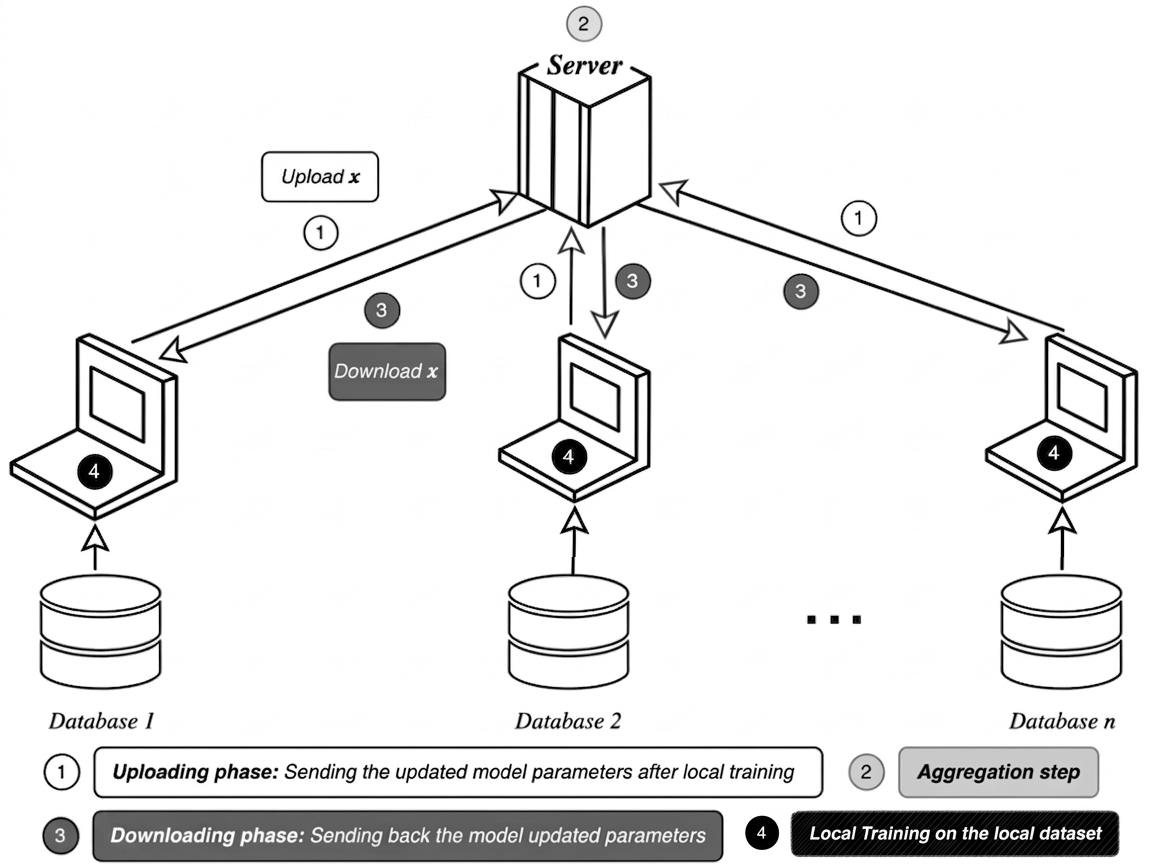}
    \caption{{\small Federated Learning structure where $\vect{x}$ represents the set of server's parameters}.}
    \label{fig::FL}
\end{figure}

\begin{table}[t]
\centering
\caption{{\small Total upload time for $K=500$ rounds, $d=1{,}000$ parameters, $N=20$ agents. Battery budget: $1{,}200$~s. $\dagger$ indicates budget violation.}}
\label{tab::comm_low}
\scriptsize
\renewcommand{\arraystretch}{1.2}
\begin{tabular}{ccccc}
\hline
\textbf{Uplink} & \textbf{Upload} & \multicolumn{2}{c}{\textbf{Total Upload Time} ($K=500$)} \\
\cline{3-4}
\textbf{Bandwidth} & \textbf{Time/Round} & \textbf{Concurrent} & \textbf{TDMA} ($N{=}20$) \\
\hline
$1$~kbps    & $32$~s    & $16{,}000$~s ($4.4$~h)$^\dagger$  & $320{,}000$~s ($88.9$~h)$^\dagger$ \\
$10$~kbps   & $3.2$~s   & $1{,}600$~s ($26.7$~min)$^\dagger$ & $32{,}000$~s ($8.9$~h)$^\dagger$  \\
$50$~kbps   & $0.64$~s  & $320$~s ($5.3$~min)                & $6{,}400$~s ($106.7$~min)$^\dagger$ \\
$100$~kbps  & $0.32$~s  & $160$~s ($2.7$~min)                & $3{,}200$~s ($53.3$~min)$^\dagger$ \\
\hline
\end{tabular}
\end{table}

Despite its success, the primary bottleneck in scaling FL remains the communication overhead between clients and the server~\cite{li2020federated, wang2021field}. As illustrated in Fig.~\ref{fig::FL}, the iterative nature of FL requires repeated download and upload phases. Recent empirical studies suggest that in communication-constrained environments, the time spent on \emph{model parameter upload} can exceed local computation time by orders of magnitude~\cite{li2020federated, wang2021field}. Consider for example a network of $N = 20$ battery-limited edge devices, each with a $20$-minute ($1{,}200$-second) operational budget, collaboratively training a model with $d = 1{,}000$ parameters over a low-power wide-area network~(LPWAN). Table~\ref{tab::comm_low} reports the total upload time over $K = 500$ rounds. At this scale, standard FL becomes nearly impossible for the majority of LPWAN operating points. For instance, under Time Division Multiple Access~(TDMA) scheduling~\cite{goldsmith2005wireless},
where agents transmit sequentially in dedicated time slots, at $10$~kbps, the upload phase alone requires over $8$ hours---exceeding the battery budget by a factor of $26$. Even at $100$~kbps, the scheduling overhead of TDMA consumes nearly an hour, rendering real-time learning and rapid adaptation infeasible.

This bottleneck is severely exacerbated as model complexity grows. In modern ``cross-device'' FL settings, such as
Google's Gboard or Apple's Siri, the system may involve millions of potential
participating devices~\cite{bonawitz2019towards, kairouz2021advances}. Similarity, in large-scale data-driven control applications, a network of embedded agents---such as autonomous drone swarms or robotic manipulators---may collaboratively train a deep neural network~(DNN) controller with $d \approx 10^6$ parameters~\cite{kaufmann2023champion, kownacki2025applying}. In such settings, each agent must upload $32$~Mbits per round at 32-bit precision. Under TDMA with $N = 100$ agents, even a high-end $1$~Gbps link accumulates $3{,}200$~s ($53.3$~min) of upload time over $K = 1{,}000$ rounds, exceeding the $20$-minute mission budget. At a more common $100$~Mbps 4G/LTE uplink, TDMA requires $8.9$~h$^\dagger$, and at $10$~Mbps it reaches $88.9$~h$^\dagger$---more than three orders of magnitude over budget. 

In both scenarios, the high dimensionality of model parameters combined with restricted upload bandwidth leads to prohibitive transmission delays. Consequently, \emph{wall-clock time-to-accuracy}---rather than the total number of iterations---has become the gold-standard metric for evaluating FL efficiency~\cite{li2020federated, wang2021field}.

Strategies to reduce FL communication overhead fall into
three main classes. The first, \emph{local computation and
client selection}, reduces communication frequency through
multiple local SGD steps~\cite{HBM-EM-DR-SH-BAA:17,
karimireddy2020scaffold} or by activating only a subset of
clients per round~\cite{yang2019scheduling, chen2018lag,
rostami2023federated}, complemented by adaptive resource
management schemes~\cite{chen2021communication} and
communication-efficient dual coordinate
methods~\cite{jaggi2014communication, zhang2015disco}.
While effective at reducing the number of communication
events, each participating client must still transmit a
full $d$-dimensional gradient or model vector, which
remains prohibitive for resource-limited edge devices.
The second class, \emph{model compression}, constrains
the global model size through techniques such as local
representation learning~\cite{liang2020think} or low-rank
adaptation~\cite{hu2022lora}, but such structural
constraints limit model expressiveness and may degrade
task performance in general settings. The third and most
directly related class, \emph{gradient compression},
reduces the size of each transmitted update. Quantization
methods such as \texttt{QSGD}~\cite{alistarh2017qsgd},
\texttt{UVeQFed}~\cite{shlezinger2020uveqfed}, and
\texttt{FedPAQ}~\cite{reisizadeh2020fedpaq} represent
gradients at reduced precision, while sparsification
methods transmit only the most informative gradient
components~\cite{lin2017deep, ivkin2019communication}.
Structured and sketched updates compress gradients onto
pre-defined subspaces via projection
matrices~\cite{park2023regulated, guo2024low} or
Count-Sketch~\cite{ivkin2019communication,
rothchild2020fetchsgd}. Critically, all of these methods
rely on \emph{static} projections shared across all
clients and rounds, and none achieves a per-round upload
cost that is independent of $d$.

In this paper, we propose \texttt{FedScalar}, a novel
communication-efficient FL algorithm that breaks the
$\mathcal{O}(d)$ upload barrier by reducing the per-round
communication payload to just two scalar values per agent,
regardless of model dimension $d$. The core idea is to
encode each agent's local update into a single scalar via
a random projection, transmitting only this scalar and a
compact seed (fixed-width integer, 32 bits) that allows the server to reconstruct an
unbiased gradient estimate without any high-dimensional
exchange. While our theoretical analysis establishes a convergence
rate of $\mathcal{O}(d/\sqrt{K})$ for smooth non-convex
functions---which carries a dimension-dependent factor $d$
in comparison to the $\mathcal{O}(1/\sqrt{K})$ rate of
standard \texttt{FedAvg}---this comparison is not directly
applicable under a fixed wall-clock budget: by reducing
the per-round cost from $\mathcal{O}(d)$ to
$\mathcal{O}(1)$, \texttt{FedScalar} affords proportionally
more rounds within the same time window, yielding a net
convergence advantage in bandwidth-constrained settings. We further show that
adopting a Rademacher distribution for the random
projection significantly reduces aggregation variance
compared to the Gaussian case, and validate both claims
through numerical simulations on a classification benchmark. 

\emph{Notations, terminologies and assumptions}: We let $\reals$, $\realpositive$, $\realnonnegative$ denote the set of real, positive real, and non-negative real numbers, respectively. For $\vect{x}\in\reals^d$, $\|\vect{x}\|=\sqrt{\vect{x}^\top\vect{x}}$ denotes the Euclidean norm, and $\langle \vect{x}, \vect{y} \rangle$ denotes the inner product. A function $f: \real^d \rightarrow \real$ is $\mathsf{L}$-Lipschitz smooth if $\|\nabla f(\vect{x})-\nabla f(\vect{y})\| \leq \mathsf{L} \|\vect{x}-\vect{y}\|$. We utilize Jensen's inequality: $\EX[\|\frac{1}{N}\sum_{n=1}^{N}\vect{x}_n\|^2] \leq \frac{1}{N}\sum_{n=1}^{N} \EX[\|\vect{x}_n\|^2]$.
Throughout this paper, we make the following assumptions about the cost function in \eqref{eq::sum}.
\begin{assump}[$\mathsf{L}$-smoothness of the cost function]\label{assump::smoothness}
The cost function in~\eqref{eq::sum} has $\mathsf{L}$-Lipschitz continuous gradient with constant $\mathsf{L}\in\real_{>0}$, i.e., 
\begin{equation}
    \|\nabla f(\vect{x}) - \nabla f(\vect{{\hat{\vect{x}}}})\| \leq \mathsf{L}\|\vect{x} - \hat{\vect{x}}\|
\end{equation}
    for any $\vect{x}, \hat{\vect{x}} \in\real^d$.
    \boxend
\end{assump}

\begin{algorithm}[t]
\caption{\textsc{FedScalar Algorithm
}}
\label{alg2_new}
{\small
\begin{algorithmic}[1]

\State \textbf{Server executes:}
\State \textbf{Input} $\vect{x}_0 \in \mathbf{R}^d$, rounds $K$
\For{each round $k = 0, 1, \ldots, K - 1$}

    \For{each client $n \in N$ \textbf{in parallel}}
        \State $r_{n}^k, \xi_{k,n} \leftarrow$ \textsc{ClientStage}$(n, \vect{x}_k)$
    \EndFor

        \State $\vect{\Delta}_{\text{sum}} \gets \vect{0}_d$  \Comment{reset the estimator}
    
    \For{each client \( n \in N\)} 
    \vspace{0.1in}
    \State Server generates vector \( \vect{v}_{k,n}  \sim  \mathcal{N}(\vect{0}_{d},\, \vect{I}_d) \) using seed \( \xi_{k,n} \) 

    \vspace{0.035in}
    \State   $\vect{\Delta}_{\text{sum}} 
         \;\gets\;
         \vect{\Delta}_{\text{sum}}  
         + 
         r_n^k\,\vect{v}_{k,n}$
 \EndFor
  
    \State \textbf{Aggregate:}\; $ \displaystyle \hat{\vect g}(\vect{x}_k) = \frac{1}{ N}  \vect{\Delta}_{\text{sum}} $
    \State  \( \vect{x}_{k+1} = \vect{x}_k + \hat{\vect{g}}(\vect{x}_k) \)
\EndFor
\smallskip
\smallskip
\Procedure{ClientStage}{$n, \vect{x}$} 
    \State $\vect{\psi}_{k,0}^n = \vect{x}$
\State{Sample $\vect{v}_{k,n} \sim \mathcal{N}(\vect{0}_{d},\, \vect{I}_d)$ using seed $\xi_{k,n}$}
    \For {s $\leftarrow 0,..., \emph{S} - 1$ }
        \State $\vect{\psi}_{k,s + 1}^n = \vect{\psi}_{k,s}^n - \alpha h_n(\vect{\psi}_{k,s}^n)$
    \EndFor
    \State $\delta_n^k = \vect{\psi}_{k, S}^n - \vect{\psi}_{k, 0}^n$
    \smallskip
    \State {$r_n^k = \langle \delta_n^{k}, \vect{v}_{k,n}\rangle$}
    \State \textbf{return} $r_n^k \in \real$ and $ \xi_{k,n} \in \mathbb{Z}$ to server
\EndProcedure
\end{algorithmic}
}
\end{algorithm}

\section{Communication-Efficient Federated Learning}\label{com-eff}
The central bottleneck in standard FL is that each active agent must upload a $d$-dimensional vector to the server every round, which is prohibitive when $d$ is large. The key idea behind our proposed \texttt{FedScalar}  algorithm (Algorithm~\ref{alg2_new}) is to replace this high-dimensional upload with a single scalar, without losing the ability to construct an unbiased estimate of the gradient at the server. To achieve this, each agent $n$ encodes its local update difference $\delta_n^k = \boldsymbol{\psi}_{k,S}^n - \boldsymbol{\psi}_{k,0}^n$---the shift in the local model after $S$ steps of local SGD---into a scalar by taking its inner product with a shared random vector $\mathbf{v}_{k,n} \sim \mathcal{N}(\mathbf{0}_d, \mathbf{I}_d)$ at line 23 of Algorithm~\ref{alg2_new}:
\begin{equation}
    r_n^k = \langle \delta_n^k,\, \mathbf{v}_{k,n} \rangle \in \mathbb{R}.
\end{equation}
The agent transmits only $r_n^k$ and the integer seed $\xi_{k,n}$ used to generate $\mathbf{v}_{k,n}$, making the upload payload independent of the model dimension $d$ (only two scalars uploaded). At the server, the received scalar is decoded by projecting it back onto the regenerated random vector. Because the server holds the seed $\xi_{k,n}$, it can independently regenerate the identical $\mathbf{v}_{k,n}$ without any additional communication. The server then forms the reconstructed update contribution $r_n^k \mathbf{v}_{k,n} \in \mathbb{R}^d$ and aggregates across all active agents:
\begin{equation}
    \hat{\mathbf{g}}(\mathbf{x}_k) = \frac{1}{N} \sum\nolimits_{n \in N} r_n^k \mathbf{v}_{k,n},
\end{equation}
The global model is then updated as $\mathbf{x}_{k+1} = \mathbf{x}_k + \hat{\mathbf{g}}(\mathbf{x}_k)$.

The correctness of this encoding-decoding scheme rests on a fundamental property of projected directional derivatives along random vectors, which we establish next.

\begin{lem}[Unbiasedness of the projected directional derivative along a random vector $\mathbf{v}$~\cite{rostami2024forward}]\label{lem::unbias}
Let $\mathbf{v} \in \mathbb{R}^d$ be a random vector with each entry $v_i$ being independent and identically distributed with zero mean and unit variance. Then, the projected directional derivative along $\mathbf{v}$ is an unbiased estimate of $\nabla f(\mathbf{x})$, i.e.,
\begin{equation}
    \mathbb{E}\bigl[\langle \mathbf{v}, \nabla f(\mathbf{x})\rangle \mathbf{v}\bigr] = \nabla f(\mathbf{x}).
\end{equation}
\end{lem}

Lemma~\ref{lem::unbias} guarantees that the scalar encoding preserves gradient information in expectation: the projection $r_n^k \mathbf{v}_{k,n} = \langle \delta_n^k, \mathbf{v}_{k,n}\rangle \mathbf{v}_{k,n}$, with $\mathbf{v}_{k,n}$ sampled independently of the local update $\delta_n^k$, is an unbiased estimator of $\delta_n^k$, and hence the aggregated update $\hat{\mathbf{g}}(\mathbf{x}_k)$ provides an unbiased signal for descent. The following lemma additionally bounds the second moment of this estimator, which directly controls the variance introduced by the projection step and will appear explicitly in the convergence analysis.

\begin{lem}[Upper bound on the projected directional derivative along a random vector $\mathbf{v}$~\cite{nesterov2017random}]\label{upper_bound_forward}
Let $\mathbf{v} \in \mathbb{R}^d$ be a random vector with each entry $v_i$ being independent and identically distributed according to $\mathcal{N}(0,1)$. Then,
\begin{equation}
    \mathbb{E}\bigl[\|\langle \mathbf{v}, \nabla f(\mathbf{x})\rangle \mathbf{v}\|^2\bigr] \leq (d+4)\|\nabla f(\mathbf{x})\|^2.
\end{equation}
\end{lem}

Lemma~\ref{upper_bound_forward} reveals a dimension-dependent factor $d$ in the variance of the projection estimator. This is the origin of the $\mathcal{O}(d/\sqrt{K})$ convergence rate established in Theorem~\ref{thm::convergence} below, and motivates the variance reduction strategy discussed~below.

We also require the following standard assumptions on the stochastic gradients $ h_n$ used during local training. 

\begin{assump}[Unbiased stochastic gradients]\label{Assump:2}
For every stochastic gradient $h_n$ of $f_n$, for any $\mathbf{x} \in \mathbb{R}^d$ we have
\begin{equation}
    \mathbb{E}\bigl[ h_n(\mathbf{x}) \mid \mathbf{x}\bigr] = \nabla f_n(\mathbf{x}), \quad n \in \{1,\ldots,N\}.
\end{equation}
\end{assump}

\begin{assump}[Bounded stochastic gradients]\label{eq:bound}
For every stochastic gradient $h_n$ of $f_n$, there exists $G \in \mathbb{R}_{>0}$ such that
\begin{equation}
    \mathbb{E}\bigl[\|h_n(\mathbf{x})\|^2\bigr] \leq G^2, \quad n \in \{1,\ldots,N\},
\end{equation}
for any $\mathbf{x} \in \mathbb{R}^d$.
\end{assump}

These assumptions are standard in the FL literature~\cite{rostami2023federated, liu2020improved}. Together with the $L$-smoothness of $f$ (Assumption~1), they enable the following convergence guarantee for \texttt{FedScalar}.

 \begin{thm}[Convergence bound of Algorithm \ref{alg2_new} for non-convex loss functions]\label{thm::convergence}
Let $\alpha = \frac{1}{\sqrt{K}}$, and  Assumptions \ref{assump::smoothness}, \ref{Assump:2} and \ref{eq:bound} hold. Then Algorithm~\ref{alg2_new} converges to the stationary point of problem \eqref{eq::sum} with a rate of $O(d/\sqrt{K})$, satisfying the following upper bound
\begin{align}\label{result:thm1_1}
    \frac{1}{K} \sum\nolimits_{k=0}^{K-1} \EX[\|\nabla &f(\vect{x}_{k})\|^2]  \leq \frac{2}{ \sqrt{K} S} (f(\vect{x}_{0}) - f^\star)  \nonumber \\ &+ \frac{\mathsf{L}^2 S^2 G^2}{K}   
    +\frac{\mathsf{L}(d + 4)S G^2}{\sqrt{K}}.
\end{align}
where $f^\star$ is the optimal solution to \eqref{eq::sum}.\boxend
\end{thm}
The proof of Theorem \ref{thm::convergence} is presented in the appendix. 
The first term in the bound~\eqref{result:thm1_1} can be regarded as a the optimality gap, which diminishes by increase $K$ and $S$. On the other hand, the second and third terms on the left-hand side of the bound in~\eqref{result:thm1_1} correspond to the variance arising from the local SGD update steps during the \emph{ClientStage} phase and the projection onto the random vector $\vect{v}$, respectively. These observations can be deduced from the analysis of the proof and the results in Lemma~\ref{lemma:var_v}. Interestingly, unlike the first term in the convergence bound~\eqref{result:thm1_1}, these two terms increase with the number of total SGD update steps, $S$. Specifically, while the optimality gap decreases on the order of $O(1/S)$, the variance-related terms increase on the orders of $O(S^2)$ and $O(S)$, respectively. The magnitude of these variance-related error terms can be mitigated by incorporating appropriate variance reduction techniques, which will be discussed in the next subsection. 

While the convergence rate $\mathcal{O}(d/\sqrt{K})$
in~\eqref{result:thm1_1} carries a dimension-dependent
factor $d$ in comparison to the
$\mathcal{O}(1/\sqrt{K})$ rate of standard
\texttt{FedAvg}, this comparison is not directly
applicable under a fixed wall-clock time budget $T$.
In practice, per-round transmission delay is compounded
by uplink contention, network congestion, and scheduling
overhead among agents---effects that scale with the size
of the transmitted payload. Since \texttt{FedScalar}
reduces the per-round upload from $\mathcal{O}(d)$ bits
to just two scalars, it can complete substantially more
rounds within the same wall-clock budget, effectively
mitigating the dimension-dependent factor in the
convergence bound. To fully eliminate the residual
$d$-dependence, one possible approach is to transmit
a small number $m \ll d$ of independent projections
per agent, recovering a dimension-free
$\mathcal{O}(1/\sqrt{K})$ rate at a modest
$\mathcal{O}(m)$ upload cost; we leave the systematic
study of this projection-variance tradeoff to future work.

\subsection{Variance Reduction}
As established in the preceding remarks, the two variance-related
terms in the convergence bound~\eqref{result:thm1_1} can be mitigated
by targeting their respective sources. To reduce the variance
introduced during the local SGD updates, existing variance reduction
methods such as SVRG or SAG~\cite{johnson2013accelerating,
schmidt2017minimizing} can be incorporated inside \emph{FedScalar}.
However, for the sake of brevity, we do not explore these methods
in this paper.

To reduce the variance due to the projection onto the random vector
$\vect{v}$, we need an alternative approach. Notice that $\vect{v}$
is a random variable drawn from a distribution with zero mean and
unit variance. However, this choice is not the only possibility.
Indeed, the choice of the underlying distribution for $\vect{v}$ is
crucial and can serve as a mechanism to reduce the subsequent
convergence error in \emph{FedScalar}. Our result below, given in
Proposition~\ref{var_reduce}, shows that if the underlying
distribution of $\vect{v}$ is chosen to be the Rademacher
distribution instead of the normal distribution, the variance
during aggregation can be reduced.

\begin{defn}[Rademacher distribution~\cite{degroot2012probability}]
Rademacher distribution for a random vector $\vect{v}$ is defined as
$\vect{v} \in \{-1, +1\}^d$ with $\mathbb{P}(v_{i} = 1) =
\mathbb{P}(v_{i} = -1) = 0.5$, $\forall i \in \{1, 2, \dots, d\}$.
\boxend
\end{defn}

Note that the Rademacher distribution still satisfies
\begin{align} \label{prop_radem}
    \mathbb{E}[\vect{v}] = \vect{0}, \quad
    \mathbb{E}[\vect{v}\vect{v}^\top] = \mathbf{I}_d,
\end{align}
ensuring that the results in Lemma~\ref{lem::unbias}
and~\ref{upper_bound_forward} remain valid. Although the variances
of the Rademacher and standard normal distributions are both $1$,
their collective behavior in the context of sums can exhibit
different characteristics. In particular, the bounded nature of
Rademacher variables makes them less susceptible to extreme outliers
compared to the normal distribution, and certain sum or averaging
processes using Rademacher variables can exhibit reduced variance
due to the properties of finite sets and discrete
outcomes~\cite{degroot2012probability}.

\begin{prop}[Reducing the variance by changing the distribution of
$\vect{v}_{k,n}$ in \emph{FedScalar}] \label{var_reduce}
Consider Algorithm~\ref{alg2_new} where the random vector
$\vect{v}_{k,n}$ is sampled from the Rademacher distribution, i.e.,
$\vect{v}_{k,n} \in \{-1, +1\}^d$ with $\mathbb{P}(v_{k,n,i} = 1)
= \mathbb{P}(v_{k,n,i} = -1) = 0.5$, $\forall i \in \{1, 2, \dots,
d\}$. Then, the variance in the aggregation step of
\emph{FedScalar} decreases by $\frac{2}{N^2} \sum_{n=1}^N
\|\delta_n^k\|^2 \mathbf{I}_d$, i.e.,
\begin{align}
    \mathrm{Var}_{\vect{v}_{k,n} \sim \mathcal{N}(\mathbf{0},
    \mathbf{I}_d)}[\mathbf{d}_{\vect{x}_k}] &-
    \mathrm{Var}_{\vect{v}_{k,n} \sim
    \mathrm{Rademacher}^d}[\mathbf{d}_{\vect{x}_k}] \nonumber \\
    &= \frac{2}{N^2} \sum\nolimits_{n=1}^N \|\delta_n^k\|^2
    \mathbf{I}_d,
\end{align}
where $\mathbf{d}_{\vect{x}_k} = \vect{x}_{k+1} - \vect{x}_k$.
\boxend
\end{prop}

The proof of Proposition~\ref{var_reduce} is presented in the
appendix. According to this result, changing the underlying
distribution of $\vect{v}$ to the Rademacher distribution reduces
the aggregation variance by $\frac{2}{N^2} \sum_{n=1}^N
\|\delta_n^k\|^2 \mathbf{I}_d$, where $\|\delta_n^k\|^2$ grows with
the number of local SGD steps $S$ performed during the
\emph{ClientStage} phase. This motivates the use of the Rademacher
distribution in practice, particularly when $S$ is large.

\section{Numerical Simulations}\label{sec::num}
We evaluate \emph{FedScalar} (Algorithm~\ref{alg2_new}) on multiclass classification using the \emph{Digits} dataset from \texttt{sklearn}, which consists of $8\times 8$ grayscale images and thus $64$ input features. We employ a neural network with two hidden layers of widths $24$ and $12$, resulting in a model with approximately $d\approx 2000$ trainable parameters. The dataset is distributed across $N=20$ agents. In the experiment, we perform $K=1500$ communication rounds with $S=5$ local steps, batch size $32$, and stepsize $\alpha=0.003$. Results are averaged over $10$ runs. We compare \emph{FedScalar} against two standard baselines: \emph{FedAvg}~\cite{HBM-EM-DR-SH-BAA:17} and the 8-bit quantization-based \emph{QSGD}~\cite{alistarh2017qsgd}.

\begin{figure}[t]
\centering
    \includegraphics[scale=0.25]{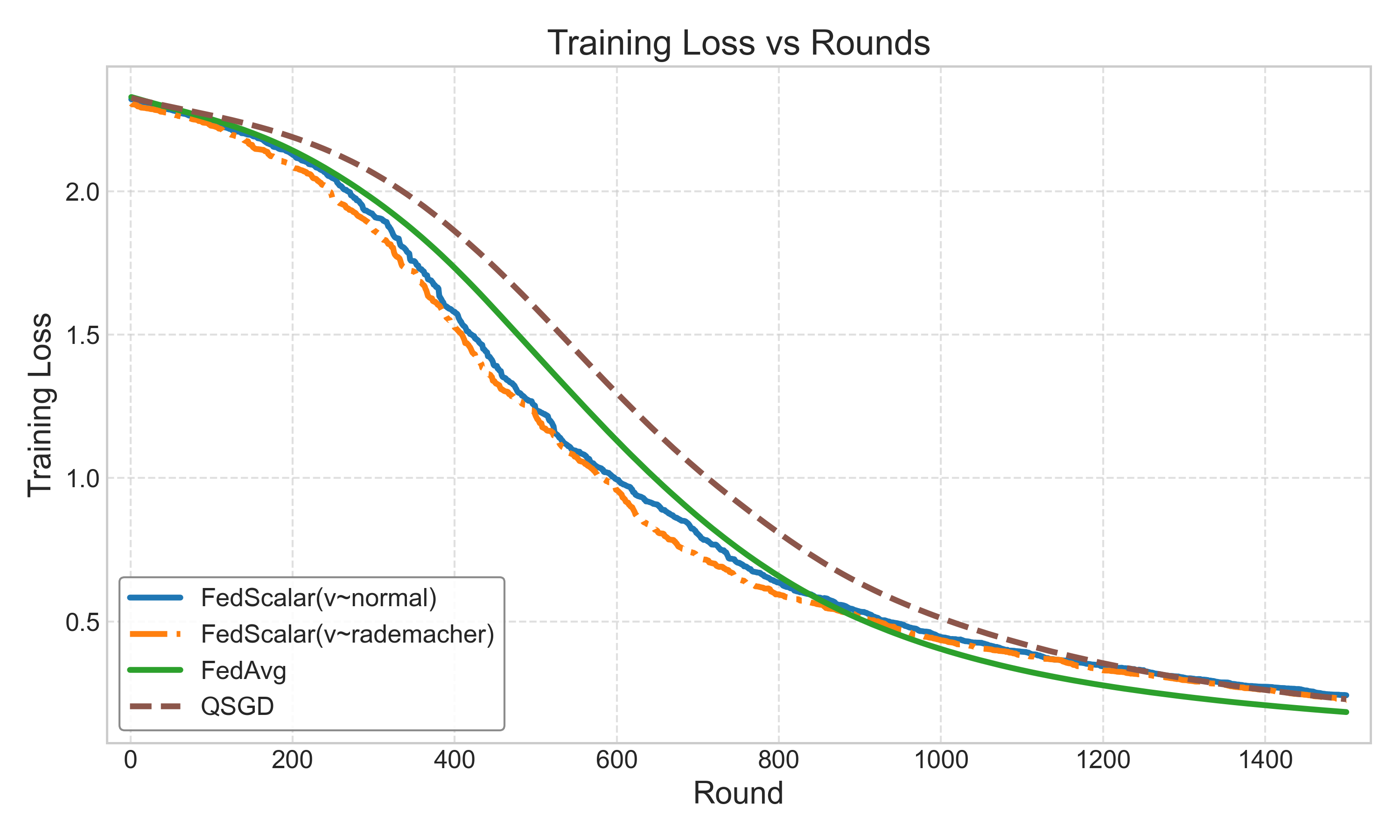}
    \caption{{\small Training loss vs.\ round (iteration)  for Algorithm~\ref{alg2_new} with $\vect{v}_k$ sampled from a normal and a Rademacher distribution, compared with \emph{FedAvg} and \emph{QSGD}.}}
    \label{loss_plot}
\end{figure} 

\begin{figure}[t]
\centering
    \includegraphics[scale=0.25]{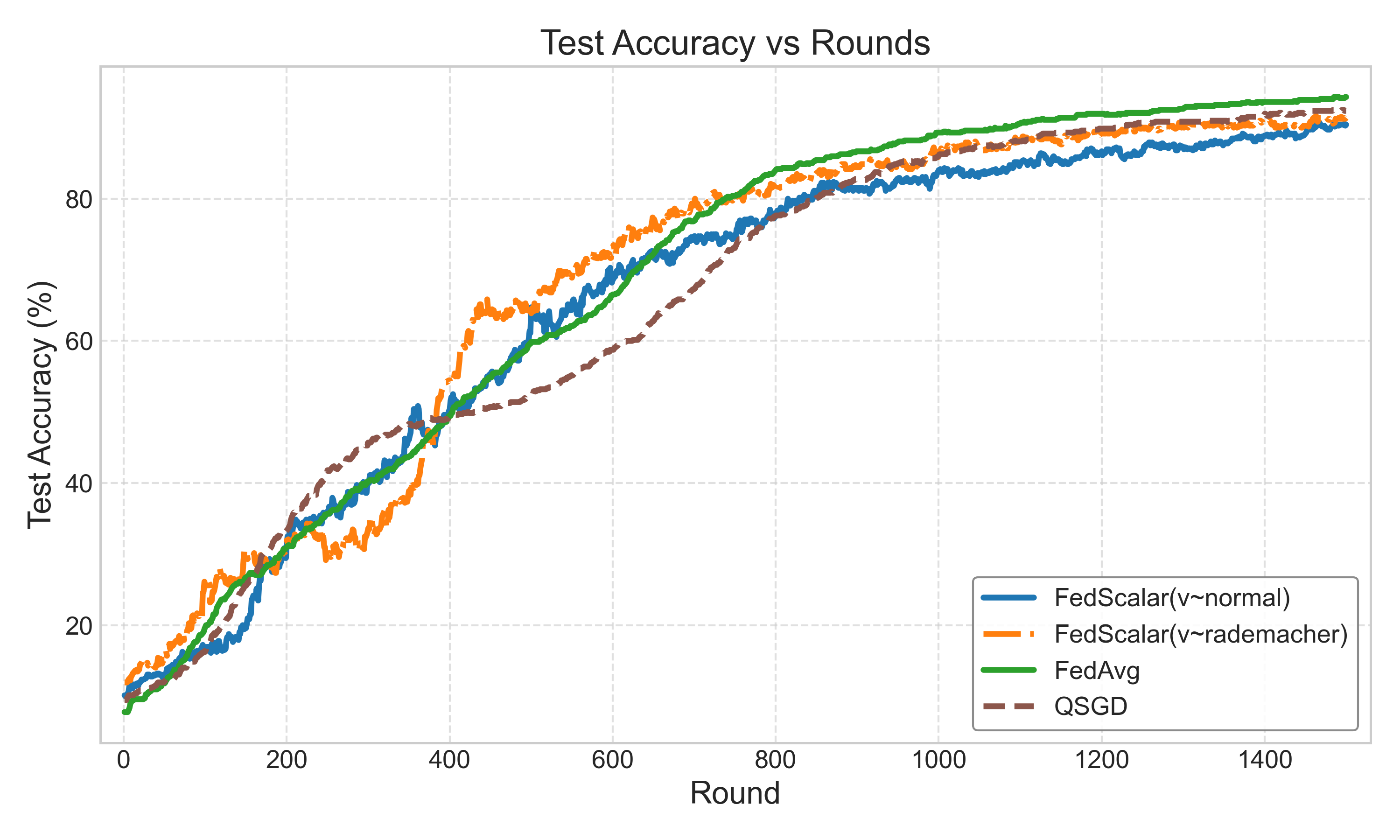}
    \caption{{\small Test accuracy vs.\ round (iteration) for  Algorithm~\ref{alg2_new} with $\vect{v}_k$ sampled from a normal and a Rademacher distribution, compared with \emph{FedAvg} and \emph{QSGD}.}}
    \label{acc_plot}
\end{figure} 

To capture system-level performance, we model both wall-clock time and communication energy. The per-round time is given by
\begin{equation}
T_{\text{wall}}^{(k)} = T_{\text{other}}^{(k)} + \frac{B_{\text{upload}}^{(k)}}{R^{(k)}},
\end{equation}

where $B_{\text{upload}}$ denotes the number of transmitted bits, and $R$ represents the uplink bandwidth (in bits per second). In all experiments, we assume $32$ bits are used to represent each floating-point value when computing $B_{\text{upload}}$. The term $B_{\text{upload}}/R$ corresponds to the upload time, while $T_{\text{other}}$ accounts for additional delays such as local computation and system overhead. In our experiments, we set the nominal uplink bandwidth to $0.1$ Mbps, corresponding to a weak or bandwidth-constrained communication regime typical in edge networks, and incorporate multiplicative log-normal variability to reflect realistic channel fluctuations. We model $T_{\text{other}}$ as a fraction of the FedAvg upload time. This modeling choice captures bandwidth-constrained edge scenarios, where communication dominates the end-to-end latency. It allows us to directly quantify the system-level benefits of communication-efficient methods, emphasizing the wall-clock time improvements achieved by \emph{FedScalar} in low-capacity regimes.

We further quantify energy consumption using a standard communication energy model \cite{bjornson2018energy}:
\begin{equation}\label{enegy_eq}
E_{\text{round}} = P_{\text{tx}} \cdot \frac{B_{\text{upload}}}{R},
\end{equation}
where $P_{\text{tx}}$ is the transmit power. In our setup, we set $P_{\text{tx}} = 2$ Watts, which is representative of energy usage in low-power edge devices. This model captures the dominant energy cost associated with wireless transmission during federated learning.

\begin{figure}[t]
\centering
    \includegraphics[scale=0.25]{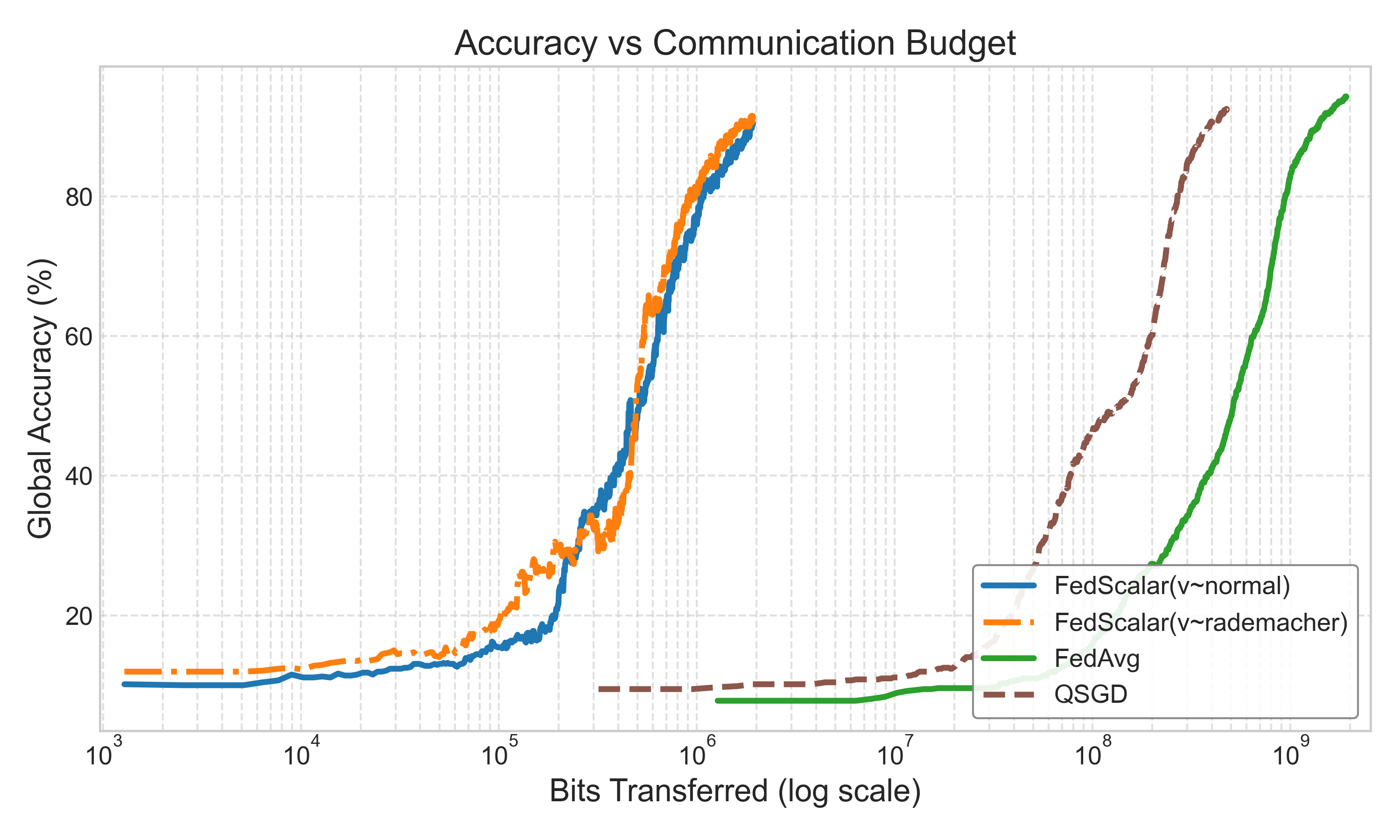}
    \caption{{\small  Accuracy vs. communication budget. The x-axis shows the cumulative number of bits transmitted to the server by all $N=20$ clients over $K=1500$ communication rounds. We compare \emph{FedScalar} (normal and Rademacher sampling) with \emph{FedAvg} and \emph{QSGD}.}}
    \label{accuracy_vs_bits}
\end{figure} 

\begin{figure}[t]
\centering
    \includegraphics[scale=0.25]{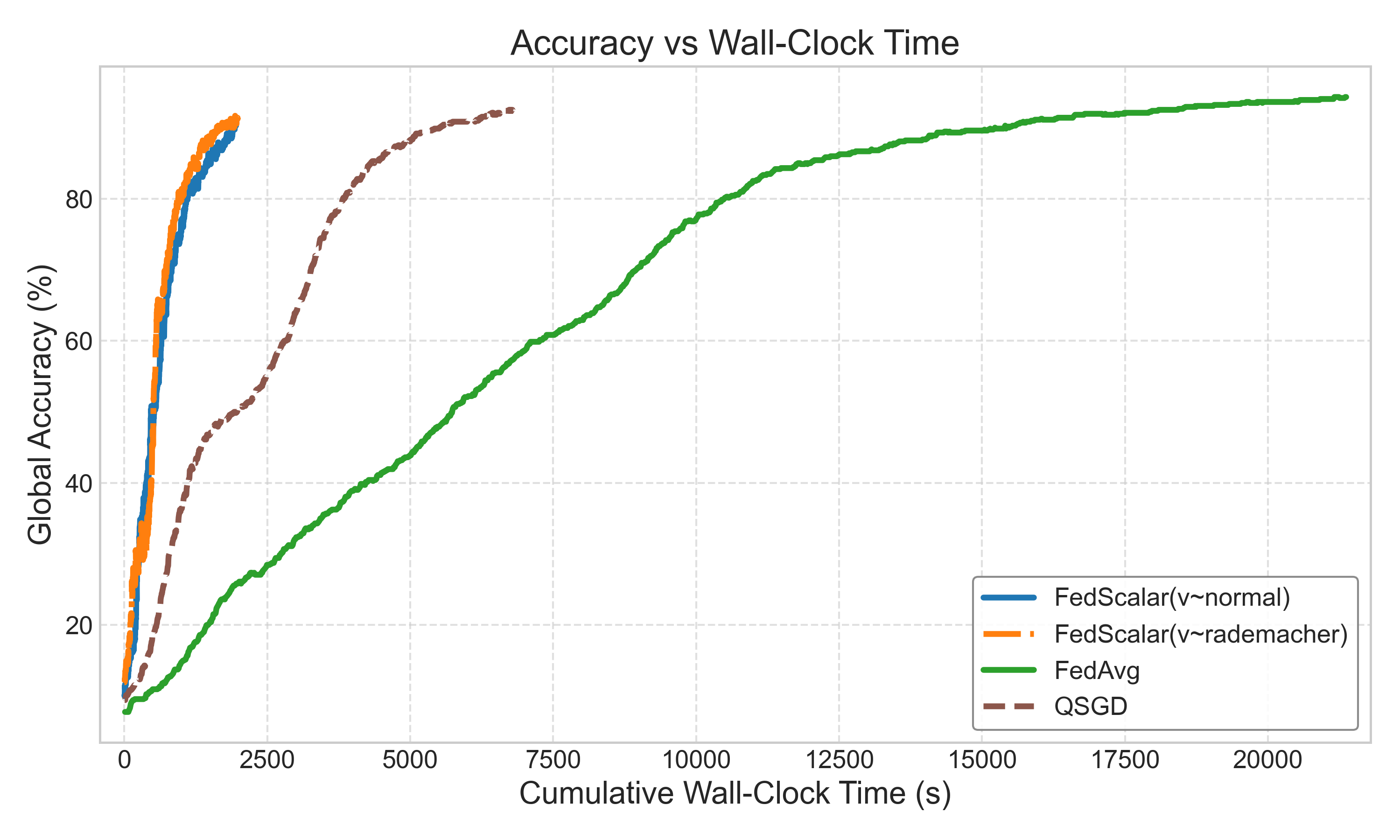}
    \caption{{\small Accuracy vs. wall-clock time. The x-axis shows the cumulative runtime across $K=1500$ communication rounds. We compare \emph{FedScalar} (normal and Rademacher sampling) with \emph{FedAvg} and \emph{QSGD}.}}
    \label{accuracy_vs_time}
\end{figure} 

\begin{figure}[t]
\centering
    \includegraphics[scale=0.25]{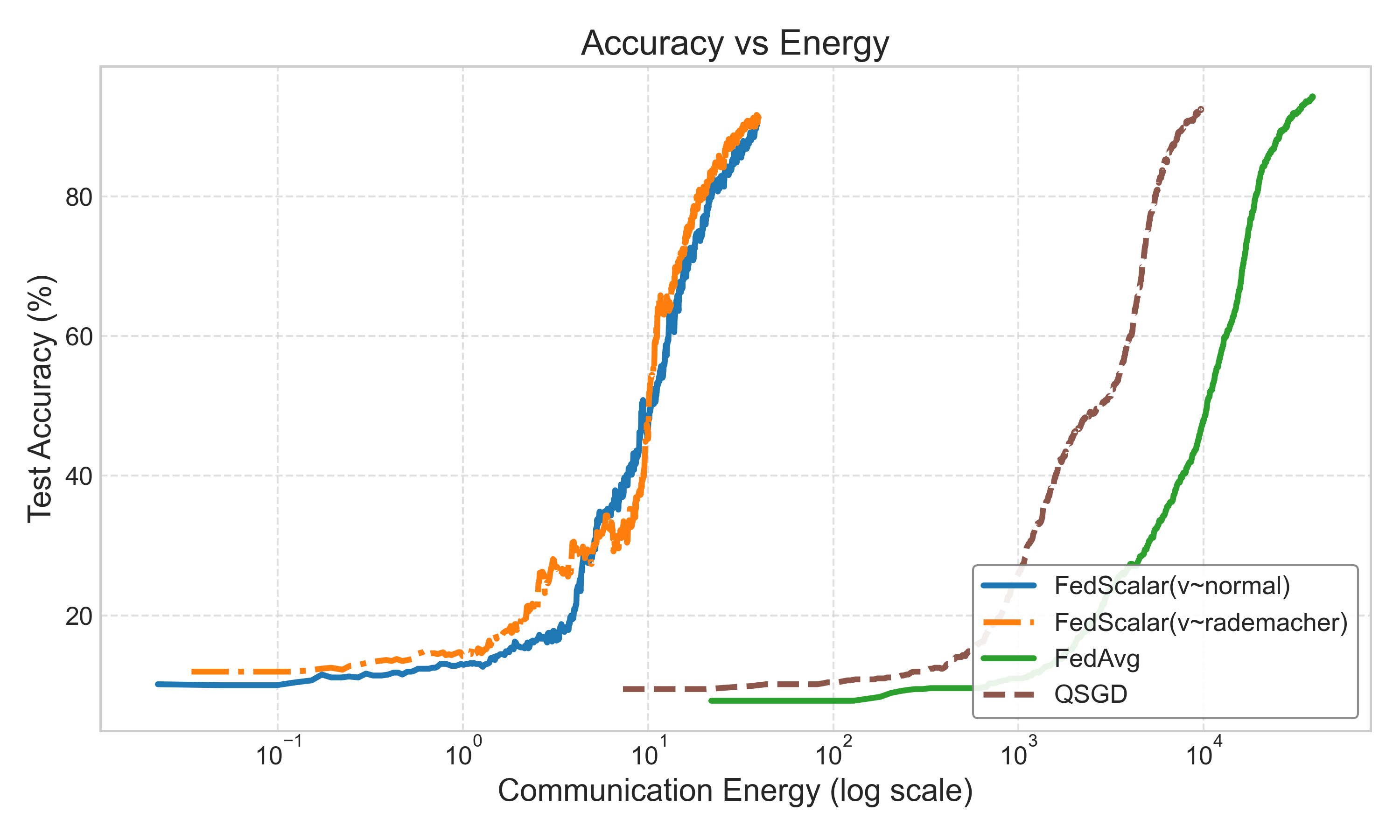}
    \caption{{\small Test accuracy vs.\ communication energy (log scale) for Algorithm~\ref{alg2_new} with $\vect{v}_k$ sampled from normal and Rademacher distributions, compared with \emph{FedAvg} and \emph{QSGD}.}}
    \label{accuracy_vs_energy}
\end{figure}

First, Figs.~\ref{loss_plot} and \ref{acc_plot} confirm that \emph{FedScalar} maintains strong optimization performance, achieving rapid convergence in both training loss and test accuracy. Moreover, the variant of \emph{FedScalar} with $\vect{v}_k$ sampled from a Rademacher distribution consistently outperforms the Gaussian variant, which aligns with the variance reduction effect established in Proposition~\ref{var_reduce}.

Building on this, Figs.~\ref{accuracy_vs_bits}, \ref{accuracy_vs_time}, and \ref{accuracy_vs_energy} highlight the primary advantages of the proposed \emph{FedScalar} algorithm in terms of communication, latency, and energy efficiency. From Fig.~\ref{accuracy_vs_bits}, \emph{FedScalar} achieves over $90\%$ test accuracy using approximately $10^5$--$10^6$ transmitted bits, whereas \emph{FedAvg} and \emph{QSGD} require on the order of $10^8$--$10^9$ bits to reach comparable accuracy levels. To provide a direct comparison, at a communication budget of approximately $10^6$ bits, \emph{FedScalar} already attains above ${90\%}$ accuracy, while both \emph{FedAvg} and \emph{QSGD} remain below ${10\%}$. Notably, for \emph{FedAvg}, this budget is insufficient to even transmit a single full model update per client (of dimension $d\approx 2000$), which severely limits its ability to make progress under such a constraint. This significant reduction stems from the communication structure of the methods: \emph{FedAvg} transmits a full $d$-dimensional model update per client (with $d\approx 2000$ in our setup), while \emph{FedScalar} communicates only two scalars per client (the projected value and a seed). As a result, \emph{FedScalar} replaces the transmission of a $d$-dimensional update with only two scalars, independent of the ambient dimension $d$, substantially reducing the per-client communication cost per round compared to \emph{FedAvg} and \emph{QSGD}.

Fig.~\ref{accuracy_vs_time} shows that these communication savings directly translate into wall-clock time improvements. In particular, at approximately $t \approx 1250$ seconds, \emph{FedScalar} achieves ${84.44\%}$ accuracy, while \emph{FedAvg} and \emph{QSGD} achieve only ${17.64\%}$ and ${43.33\%}$, respectively. This substantial performance gap highlights the advantage of \emph{FedScalar} in latency-constrained settings. Similarly, Fig.~\ref{accuracy_vs_energy} shows that communication reduction directly translates into energy savings, as predicted by~\eqref{enegy_eq}. Accordingly, the trends closely mirror those in Fig.~\ref{accuracy_vs_bits}. At comparable energy levels around $50$ Joules, \emph{FedScalar} achieves ${91.39\%}$ accuracy, while \emph{FedAvg} and \emph{QSGD} achieve only ${7.78\%}$ and ${10.14\%}$, respectively, demonstrating the superior energy efficiency of \emph{FedScalar} for energy-constrained, low-capacity edge devices.

\section{Conclusions}\label{sec::conclu}
We proposed \emph{FedScalar}, a communication-efficient federated learning algorithm that reduces the upload payload to two scalar values per agent per round. Each agent encodes its local update difference via an inner product with a shared random vector and transmits the resulting scalar along with the generating seed, allowing the server to reconstruct an unbiased gradient estimate without any high-dimensional transmission. We proved that \emph{FedScalar} achieves a convergence rate of $O(d/\sqrt{K})$ to a stationary point for smooth non-convex functions, and showed that adopting a Rademacher distribution for the random vector reduces the aggregation variance by $\frac{2}{N^2}\sum_{n=1}^{N}\|\delta_n^k\|^2 \mathbf{I}_d$ compared to the Gaussian case. Numerical simulations confirm that these theoretical gains translate into significant improvements in wall-clock time and energy efficiency over \emph{FedAvg} and QSGD in bandwidth-constrained settings. Future work will investigate accelerated variants of \emph{FedScalar} to address the $d$-dependent variance in the convergence rate, specifically by exploring the use of multiple random projections per agent. Additionally, we will study its application to privacy-preserving federated learning, where encoding updates into a single scalar provides a natural barrier against parameter inference attacks.

\bibliographystyle{ieeetr}%
\bibliography{bib/alias,bib/Reference} 

\begin{thebibliography}{10}

\bibitem{HBM-EM-DR-SH-BAA:17}
H.~B. McMahan, E.~Moore, D.~Ramage, S.~Hampson, and B.~A. Arcas, ``Communication-efficient learning of deep networks from decentralized data,'' in {\em International {C}onference on {A}rtificial {I}ntelligence and {S}tatistics}, (Lauderdale, {F}{L}), pp.~1273--1282, 2017.

\bibitem{kairouz2021advances}
P.~Kairouz and H.~B. McMahan, ``Advances and open problems in federated learning,'' {\em Foundations and trends in machine learning}, vol.~14, no.~1-2, pp.~1--210, 2021.

\bibitem{li2020federated}
T.~Li, A.~K. Sahu, A.~Talwalkar, and V.~Smith, ``Federated learning: Challenges, methods, and future directions,'' {\em {IEEE} Signal Processing Magazine}, vol.~37, no.~3, pp.~50--60, 2020.

\bibitem{dobbe2020toward}
R.~Dobbe, T.~Kersting, A.~Bhattacharya, C.~Wu, and C.~Tomlin, ``Toward a data-driven pipeline for federated multi-agent control,'' {\em arXiv preprint arXiv:2011.06360}, 2020.

\bibitem{wang2023model}
H.~Wang, L.~F. Toso, A.~Mitra, and J.~Anderson, ``Model-free learning with heterogeneous dynamical systems: A federated lqr approach,'' {\em arXiv preprint arXiv:2308.11743}, 2023.

\bibitem{zhuo2019federated}
H.~H. Zhuo, W.~Feng, Q.~Xu, Q.~Yang, and Y.~Lin, ``Federated reinforcement learning,'' {\em arXiv preprint arXiv:1901.08277}, 2019.

\bibitem{wang2021field}
J.~Wang and Z.~C. et~al., ``A field guide to federated optimization,'' {\em arXiv preprint arXiv:2107.06917}, 2021.

\bibitem{goldsmith2005wireless}
A.~Goldsmith, {\em Wireless Communications}.
\newblock Cambridge University Press, 2005.

\bibitem{bonawitz2019towards}
K.~Bonawitz, H.~Eichner, W.~Grieskamp, D.~Huba, A.~Ingerman, V.~Ivanov, C.~Kiddon, J.~Kone{\v{c}}n{\`y}, S.~Mazzocchi, B.~McMahan, {\em et~al.}, ``Towards federated learning at scale: System design,'' {\em Proceedings of machine learning and systems}, vol.~1, pp.~374--388, 2019.

\bibitem{kaufmann2023champion}
E.~Kaufmann, L.~Bauersfeld, A.~Loquercio, M.~M{\"u}ller, V.~Koltun, and D.~Scaramuzza, ``Champion-level drone racing using deep reinforcement learning,'' {\em Nature}, vol.~620, pp.~982--987, 2023.

\bibitem{kownacki2025applying}
C.~Kownacki, S.~Romaniuk, and M.~Derlatka, ``Applying neural networks as direct controllers in position and trajectory tracking algorithms for holonomic {UAV}s,'' {\em Scientific Reports}, vol.~15, p.~12605, 2025.

\bibitem{karimireddy2020scaffold}
S.~P. Karimireddy, S.~Kale, M.~Mohri, S.~Reddi, S.~Stich, and A.~T. Suresh, ``Scaffold: Stochastic controlled averaging for federated learning,'' in {\em International conference on machine learning}, pp.~5132--5143, 2020.

\bibitem{yang2019scheduling}
H.~H. Yang, Z.~Liu, T.~Q. Quek, and H.~V. Poor, ``Scheduling policies for federated learning in wireless networks,'' {\em IEEE transactions on communications}, vol.~68, no.~1, pp.~317--333, 2019.

\bibitem{chen2018lag}
T.~Chen, G.~Giannakis, T.~Sun, and W.~Yin, ``Lag: Lazily aggregated gradient for communication-efficient distributed learning,'' {\em Advances in neural information processing systems}, vol.~31, 2018.

\bibitem{rostami2023federated}
M.~Rostami and S.~S. Kia, ``Federated learning using variance reduced stochastic gradient for probabilistically activated agents,'' in {\em {A}merican {C}ontrol {C}onference}, pp.~861--866, 2023.

\bibitem{chen2021communication}
M.~Chen, N.~Shlezinger, H.~V. Poor, Y.~C. Eldar, and S.~Cui, ``Communication-efficient federated learning,'' {\em Proceedings of the National Academy of Sciences}, vol.~118, no.~17, p.~e2024789118, 2021.

\bibitem{jaggi2014communication}
M.~Jaggi, V.~Smith, M.~Tak{\'a}c, J.~Terhorst, S.~Krishnan, T.~Hofmann, and M.~I. Jordan, ``Communication-efficient distributed dual coordinate ascent,'' {\em Advances in neural information processing systems}, vol.~27, 2014.

\bibitem{zhang2015disco}
Y.~Zhang and X.~Lin, ``Disco: Distributed optimization for self-concordant empirical loss,'' in {\em International conference on machine learning}, pp.~362--370, 2015.

\bibitem{liang2020think}
P.~P. e.~a. Liang, ``Think locally, act globally: Federated learning with local and global representations,'' {\em arXiv preprint arXiv:2001.01523}, 2020.

\bibitem{hu2022lora}
E.~J. Hu, Y.~Shen, P.~Wallis, Z.~Allen-Zhu, Y.~Li, S.~Wang, L.~Wang, W.~Chen, {\em et~al.}, ``{L}ora: Low-rank adaptation of large language models.,'' {\em ICLR}, vol.~1, no.~2, p.~3, 2022.

\bibitem{alistarh2017qsgd}
D.~Alistarh, D.~Grubic, J.~Li, R.~Tomioka, and M.~Vojnovic, ``Qsgd: Communication-efficient sgd via gradient quantization and encoding,'' {\em Advances in neural information processing systems}, vol.~30, 2017.

\bibitem{shlezinger2020uveqfed}
N.~Shlezinger, M.~Chen, Y.~C. Eldar, H.~V. Poor, and S.~Cui, ``{UV}e{Q}{F}ed: Universal vector quantization for federated learning,'' {\em IEEE Transactions on Signal Processing}, vol.~69, pp.~500--514, 2020.

\bibitem{reisizadeh2020fedpaq}
A.~Reisizadeh, A.~Mokhtari, H.~Hassani, A.~Jadbabaie, and R.~Pedarsani, ``Fedpaq: A communication-efficient federated learning method with periodic averaging and quantization,'' in {\em International conference on artificial intelligence and statistics}, pp.~2021--2031, 2020.

\bibitem{lin2017deep}
Y.~Lin, S.~Han, H.~Mao, Y.~Wang, and W.~J. Dally, ``Deep gradient compression: Reducing the communication bandwidth for distributed training,'' {\em arXiv preprint arXiv:1712.01887}, 2017.

\bibitem{ivkin2019communication}
N.~Ivkin, D.~Rothchild, E.~Ullah, I.~Stoica, R.~Arora, {\em et~al.}, ``Communication-efficient distributed {SGD} with sketching,'' {\em Advances in Neural Information Processing Systems}, vol.~32, 2019.

\bibitem{park2023regulated}
S.~Park and W.~Choi, ``Regulated subspace projection based local model update compression for communication-efficient federated learning,'' {\em IEEE Journal on Selected Areas in Communications}, vol.~41, no.~4, pp.~964--976, 2023.

\bibitem{guo2024low}
M.~Guo, D.~Liu, O.~Simeone, and D.~Wen, ``Low-rank gradient compression with error feedback for mimo wireless federated learning,'' {\em arXiv preprint arXiv:2401.07496}, 2024.

\bibitem{rothchild2020fetchsgd}
D.~e.~a. Rothchild, ``Fetchsgd: Communication-efficient federated learning with sketching,'' in {\em International Conference on Machine Learning}, pp.~8253--8265, 2020.

\bibitem{rostami2024forward}
M.~Rostami and S.~S. Kia, ``Projected forward gradient-guided {Frank-Wolfe} algorithm via variance reduction,'' {\em {IEEE} Control Systems Letters}, vol.~8, pp.~3153--3158, 2024.

\bibitem{nesterov2017random}
Y.~Nesterov and V.~Spokoiny, ``Random gradient-free minimization of convex functions,'' {\em Foundations of Computational Mathematics}, vol.~17, no.~2, pp.~527--566, 2017.

\bibitem{liu2020improved}
Y.~Liu, Y.~Gao, and W.~Yin, ``An improved analysis of stochastic gradient descent with momentum,'' {\em Advances in Neural Information Processing Systems}, vol.~33, pp.~18261--18271, 2020.

\bibitem{johnson2013accelerating}
R.~Johnson and T.~Zhang, ``Accelerating stochastic gradient descent using predictive variance reduction,'' {\em Advances in neural information processing systems}, vol.~26, 2013.

\bibitem{schmidt2017minimizing}
M.~Schmidt, N.~Le~Roux, and F.~Bach, ``Minimizing finite sums with the stochastic average gradient,'' {\em Mathematical Programming}, vol.~162, pp.~83--112, 2017.

\bibitem{degroot2012probability}
M.~H. DeGroot and M.~J. Schervish, ``Probability and statistics, fourth,'' 2012.

\bibitem{bjornson2018energy}
E.~Bj{\"o}rnson and E.~G. Larsson, ``How energy-efficient can a wireless communication system become?,'' in {\em 2018 52nd Asilomar conference on signals, systems, and computers}, pp.~1252--1256, 2018.

\bibitem{isserlis1918formula}
L.~Isserlis, ``On a formula for the product-moment coefficient of any order of a normal frequency distribution in any number of variables,'' {\em Biometrika}, vol.~12, no.~1/2, pp.~134--139, 1918.

\end{thebibliography}

\vspace{-0.1in}
\appendix
\renewcommand{\theequation}{A.\arabic{equation}}
\renewcommand{\thethm}{A.\arabic{thm}}
\renewcommand{\thelem}{A.\arabic{lem}}
\renewcommand{\thedefn}{A.\arabic{defn}}

This section presents the proofs of the results in the paper and the auxiliary lemmas used in these proofs. 
\begin{lem}\label{lemma1}
Consider Algorithm~\ref{alg2_new}. Then, we have  
{\small
\begin{align}\label{eq::10_new}
-\EX[\big \langle \nabla f(\vect{x}_{k}), \frac{1}{N} &\sum\nolimits_{n=1}^N \nabla f_n(\vect{\psi}_{k,s}^n) \big \rangle] \leq -\frac{1}{2} \EX[\|\nabla f(\vect{x}_{k})\|^2]  \nonumber \\
&+ \frac{ \mathsf{L}^2 S \alpha^2}{2N} \sum_{n=1}^N \sum_{s^{\prime}=0}^{s-1}\EX[\| h_n(\vect{\psi}_{k,s^{\prime}}^n)\|^2].
\end{align}
}
\end{lem}

\begin{proof}
Adding and subtracting $\nabla f(\vect{x}_{k})$ results in
{\small
\begin{align}
&-\EX[\big \langle \nabla f(\vect{x}_{k}), \frac{1}{N} \sum\nolimits_{n=1}^N \nabla f_n(\vect{\psi}_{k,s}^n) \big \rangle] 
=-\EX[\big \langle\nabla f(\vect{x}_{k}), \nonumber \\
&\frac{1}{N} \sum\nolimits_{n=1}^N \nabla f_n(\vect{\psi}_{k,s}^n) +  \nabla f(\vect{x}_{k}) -  \nabla f(\vect{x}_{k}) \big \rangle] \nonumber \\
&=  \EX[\big \langle\nabla f(\vect{x}_{k}), \nabla f(\vect{x}_{k}) - \frac{1}{N} \sum_{n=1}^N \nabla f_n(\vect{\psi}_{k,s}^n) \big \rangle] \nonumber \\
& - \EX[\big \langle\nabla f(\vect{x}_{k}), \nabla f(\vect{x}_{k})\big \rangle], \nonumber \\
&\leq \frac{1}{2} \EX[\|\nabla f(\vect{x}_{k})\|^2] + \frac{1}{2} \EX [\|\nabla f(\vect{x}_{k}) \nonumber \\
&- \frac{1}{N} \sum\nolimits_{n=1}^N \nabla f_n(\vect{\psi}_{k,s}^n)\|^2] - \EX[\big \langle\nabla f(\vect{x}_{k}),  \nabla f(\vect{x}_{k})\big \rangle], \nonumber \\
&=-\frac{1}{2} \EX[\|\nabla f(\vect{x}_{k})\|^2] + \frac{1}{2} \EX [\| \frac{1}{N} \sum_{n=1}^N \nabla f_n(\vect{x}_{k}) \nonumber \\
&- \frac{1}{N}\sum_{n=1}^N \nabla f_n(\vect{\psi}_{k,s}^n)\|^2], \nonumber \\
&=-\frac{1}{2} \EX[\|\nabla f(\vect{x}_{k}) \|^2] + \frac{1}{2 N^2} \EX [\| \sum_{n=1}^N \nabla f_n(\vect{x}_{k}) \nonumber \\
&- \nabla f_n(\vect{\psi}_{k,s}^n)\|^2], \nonumber \\
&\leq-\frac{1}{2} \EX[\|\nabla f(\vect{x}_{k})\|^2] + \frac{1}{2N}\sum_{n=1}^N \EX [\| \nabla f_n(\vect{x}_{k}) \nonumber \\
&- \nabla f_n(\vect{\psi}_{k,s}^n)\|^2], \nonumber \\
&\leq-\frac{1}{2} \EX[\|\nabla f(\vect{x}_{k})\|^2] + \frac{ \mathsf{L}^2}{2N} \sum_{n=1}^N \EX[\|\vect{x}_k - \vect{\psi}_{k,s}^n \|^2],\nonumber \\
&=-\frac{1}{2} \EX[\|\nabla f(\vect{x}_{k})\|^2] + \frac{ \mathsf{L}^2}{2N} \sum_{n=1}^N \EX[\|\sum_{s^{\prime}=0}^{s-1} \alpha h_n(\vect{\psi}_{k,s^{\prime}}^n) \|^2], \nonumber \\
&\leq-\frac{1}{2} \EX[\|\nabla f(\vect{x}_{k})\|^2] + \frac{ \mathsf{L}^2 S \alpha^2}{2N} \sum_{n=1}^N \sum_{s^{\prime}=0}^{s-1}\EX[\| h_n(\vect{\psi}_{k,s^{\prime}}^n)\|^2], \nonumber 
\end{align}
}
in the second and forth inequality follows from the Jensen's inequality, and the third inequality comes from the Liptizeness of the gradient. The last inequality follows from the Jensen's inequality. 
\end{proof}

\begin{lem}\label{lemma:var_v}
    Consider Algorithm \ref{alg2_new}. Then, we have 
  {\small
    \begin{align}\label{eq:var_v}
    \frac{\mathsf{L}}{2N}\sum\nolimits_{n=1}^N& \EX\big[\| r_{n}^k \vect{v}_{k,n}\|^2\big] \leq  \\
    &\frac{\alpha^2 \mathsf{L} (d + 4) S }{2N}\sum_{n=1}^N \sum\nolimits_{s=0}^{S-1} \EX [\|h_n(\vect{\psi}_{k, s}^{n}) \|^2].\nonumber 
     \end{align}
     }
Moreover, if we invoke Assumption \ref{eq:bound}, we get 
{\small
\begin{align}\label{eq:var_v_2}
    \frac{\mathsf{L}}{2N}\sum\nolimits_{n=1}^N& \EX\big[\| r_{n}^k \vect{v}_{k,n}\|^2\big] \leq \frac{\alpha^2 \mathsf{L} (d + 4) S^2 G^2}{2}.
     \end{align}
     }
\end{lem}
\vspace{0.2in}
\begin{proof}
Invoking Lemma \ref{upper_bound_forward} results in
{\small
    \begin{align}
 &\frac{\mathsf{L}}{2N}\sum\nolimits_{n=1}^N \EX\big[\| r_{n}^k \vect{v}_{k,n}\|^2\big]  \leq \frac{\mathsf{L} (d + 4)}{2N}\sum_{n=1}^N  \EX [\| \delta_n^k \|^2] \nonumber \\
    &=\frac{\alpha^2 \mathsf{L} (d + 4)}{2N}\sum_{n=1}^N  \EX [\| \sum_{s=0}^{S-1}  h_n(\vect{\psi}_{k, s}^{n}) \|^2] \nonumber \\
    &\leq  \frac{\alpha^2 \mathsf{L} (d + 4) S }{2N}\sum_{n=1}^N \sum\nolimits_{s=0}^{S-1} \EX [\|h_n(\vect{\psi}_{k, s}^{n}) \|^2] 
\end{align}
}
where the last inequality follows from Jensen's inequality.
\end{proof}

\begin{proof}[Proof of Lemma~\ref{lem::unbias}]
The proof follows from 
$\EX[ \langle \vect{v}, \nabla f (\vect{x})\rangle \vect{v}] = \EX[\vect{v}\vect{v^\top}]\nabla f (\vect{x}) = \vect{I}_n \nabla f (\vect{x}) = \nabla f (\vect{x}),$ where we use the fact that a random vector $\vect{v}$ and $\nabla f (\vect{x})$ are independent of each other.
\end{proof}

\begin{proof}[Proof of Theorem~\ref{thm::convergence}]
    Since $f$ is $\mathsf{L}-smooth$, we can write
\begin{align}\label{eq::1}
     \EX[f(\vect{x}_{k+1})|\vect{x}_{k}] &\leq f(\vect{x}_{k})+\big \langle\nabla f(\vect{x}_{k}), \EX[\vect{x}_{k+1} - \vect{x}_{k}|\vect{x}_{k}] \big \rangle \nonumber \\
    &+\frac{\mathsf{L}}{2}\EX\big[\|\vect{x}_{k+1} - \vect{x}_{k}\|^2\big| \vect{x}_{k}\big],
\end{align}
plugging back the the result from lines 9-14 of Algorithm \ref{alg2_new} into \eqref{eq::1}, we have 
{\small
\begin{align*}
     &\EX[f(\vect{x}_{k+1})|\vect{x}_{k}] \leq f(\vect{x}_{k})+ \big \langle\nabla f(\vect{x}_{k}), \EX[\frac{1}{N} \sum\nolimits_{n=1}^N r_{n}^k \vect{v}_{k,n}|\vect{x}_{k}] \big \rangle \nonumber \\
    &+\frac{\mathsf{L}}{2}\EX\big[\|\frac{1}{N}\sum\nolimits_{n=1}^N r_{n}^k \vect{v}_{k,n}\|^2\big| \vect{x}_{k}\big],\nonumber\\
    &= f(\vect{x}_{k}) +  \big \langle\nabla f(\vect{x}_{k}), \EX[\frac{1}{N} \sum\nolimits_{n=1}^N \delta_n^k|\vect{x}_{k}] \big \rangle \nonumber \\
    &+\frac{\mathsf{L}}{2}\EX\big[\|\frac{1}{N} \sum\nolimits_{n=1}^N r_{n}^k \vect{v}_{k,n}\|^2\big| \vect{x}_{k}\big],\nonumber \\
     &= f(\vect{x}_{k}) -  \big \langle\nabla f(\vect{x}_{k}), \EX[\frac{1}{N} \sum\nolimits_{n=1}^N \sum\nolimits_{s=0}^{S-1} \alpha h_n(\vect{\psi}_{k, s}^{n}) |\vect{x}_{k}] \big \rangle \nonumber \\
    &+\frac{\mathsf{L}}{2}\EX\big[\|\frac{1}{N} \sum\nolimits_{n=1}^N r_{n}^k \vect{v}_{k,n}\|^2\big| \vect{x}_{k}\big],\nonumber \\
    &\leq f(\vect{x}_{k}) -  \big \langle\nabla f(\vect{x}_{k}), \EX[ \frac{1}{N} \sum\nolimits_{n=1}^N \sum\nolimits_{s=0}^{S-1} \alpha h_n(\vect{\psi}_{k, s}^{n}) |\vect{x}_{k}] \big \rangle \nonumber \\
    &+\frac{\mathsf{L}}{2N}\sum\nolimits_{n=1}^N \EX\big[\| r_{n}^k \vect{v}_{k,n}\|^2\big| \vect{x}_{k}\big], \\
    \end{align*}
}
where the last inequality follows from Jensen's inequality. Using Assumption \ref{Assump:2} into the last inequality, we have
{\small    \begin{align*}
    & \EX[f(\vect{x}_{k+1})|\vect{x}_{k}] \leq f(\vect{x}_{k}) \nonumber \\
     &-  \big \langle\nabla f(\vect{x}_{k}), \EX[ \frac{1}{N} \sum\nolimits_{n=1}^N \sum\nolimits_{s=0}^{S-1} \alpha \nabla f_n(\vect{\psi}_{k, s}^{n}) |\vect{x}_{k}] \big \rangle \nonumber \\
    &+\frac{\mathsf{L}}{2N}\sum\nolimits_{n=1}^N \EX\big[\| r_{n}^k \vect{v}_{k,n}\|^2\big| \vect{x}_{k}\big] \\
      &=f(\vect{x}_{k}) -  \alpha \sum\nolimits_{s=0}^{S-1} \!\!\EX [\big \langle\nabla f(\vect{x}_{k}), \frac{1}{N}\!\sum\nolimits_{n=1}^N \!\! \nabla f_n(\vect{\psi}_{k, s}^{n}) \big \rangle |\vect{x}_{k}] \nonumber \\
    &+\frac{\mathsf{L}}{2N}\sum\nolimits_{n=1}^N \EX\big[\| r_{n}^k \vect{v}_{k,n}\|^2\big| \vect{x}_{k}\big] 
\end{align*}
}
Taking total expectation on both sides we have 
{\small
\begin{align}\label{eq::9_new}
 &\EX[f(\vect{x}_{k+1})] \leq \EX[f(\vect{x}_{k})] \nonumber \\&
 -  \alpha \sum\nolimits_{s=0}^{S-1} \EX [\big \langle\nabla f(\vect{x}_{k}), \frac{1}{N} \sum\nolimits_{n=1}^N  \nabla f_n(\vect{\psi}_{k, s}^{n}) \big \rangle ] \nonumber \\
    &+\frac{\mathsf{L}}{2N}\sum\nolimits_{n=1}^N \EX\big[\| r_{n}^k \vect{v}_{k,n}\|^2\big].
\end{align}
}

Invoking Lemma \ref{lemma1} and \ref{lemma:var_v} into \eqref{eq::9_new}, we have
{\small
\begin{align}
 \EX[f&(\vect{x}_{k+1})] 
    \leq \EX[f(\vect{x}_{k})] -\frac{\alpha S}{2} \EX[\|\nabla f(\vect{x}_{k})\|^2] \nonumber \\ &+ \frac{ \mathsf{L}^2 S \alpha^3}{2N} \sum_{n=1}^N \sum_{s=0}^{S-1} \sum_{s^{\prime}=0}^{s-1}\EX[\| h_n(\vect{\psi}_{k,s^{\prime}}^n)\|^2] \nonumber \\
    &+\frac{\alpha^2 \mathsf{L} (d + 4) S }{2N}\sum_{n=1}^N \sum\nolimits_{s=0}^{S-1} \EX [\|h_n(\vect{\psi}_{k, s}^{n}) \|^2]. 
\end{align}
}
Rearranging and summing $k$ from $0$ to $K - 1$, we have 
{\small
\begin{align}
    \frac{1}{K} \sum\nolimits_{k=0}^{K-1} &\EX[\|\nabla f(\vect{x}_{k})\|^2] \leq \frac{2}{ K \alpha S}(\EX[f(\vect{x}_{0})] - \EX[f(\vect{x}_{K})])  \nonumber \\ &+ \frac{\mathsf{L}^2  \alpha^2}{NK} \sum_{k=0}^{K-1} \sum_{n=1}^N \sum_{s=0}^{S-1} \sum_{s^{\prime}=0}^{s-1} \EX[\| h_n(\vect{\psi}_{k,s^{\prime}}^n)\|^2] \nonumber \\
    &+\frac{\mathsf{L}\alpha (d + 4)}{KN}\sum_{k=0}^{K-1} \sum_{n=1}^N  \sum_{s=0}^{S-1} \| \nabla h_n(\vect{\psi}_{k,s}^n) \|^2, \nonumber \\
    &\leq \frac{2}{ K\alpha S } (f(\vect{x}_{0}) - f^\star)   \nonumber \\ &+ \frac{\mathsf{L}^2  \alpha^2}{NK} \sum_{k=0}^{K-1} \sum_{n=1}^N \sum_{s=0}^{S-1} \sum_{s^{\prime}=0}^{s-1} \EX[\| h_n(\vect{\psi}_{k,s^{\prime}}^n)\|^2] \nonumber \\
    &+\frac{\mathsf{L}\alpha (d + 4)}{KN}\sum_{k=0}^{K-1} \sum_{n=1}^N  \sum_{s=0}^{S-1} \| \nabla h_n(\vect{\psi}_{k,s}^n) \|^2, \nonumber 
\end{align}
}
Moreover, by invoking Assumption \ref{Assump:2}, we can simplify the bound as the following
\begin{align}\label{result:thm1}
    \frac{1}{K} \sum\nolimits_{k=0}^{K-1} \EX[\|\nabla &f(\vect{x}_{k})\|^2]  \leq \frac{2}{ K \alpha S} (f(\vect{x}_{0}) - f^\star)  \nonumber \\ &+ \mathsf{L}^2 S^2 \alpha^2 G^2   
    +\mathsf{L} \alpha (d + 4)S G^2.
\end{align}
If we set $\alpha = \frac{1}{\sqrt{K}}$ in \eqref{result:thm1}, we have a convergence rate of $O(\frac{d}{\sqrt{K}})$ to a stationary point of $f(\vect{x})$.
\end{proof}

\begin{proof}[Proof of Proposition~\ref{var_reduce}]
Recall the aggregation step (line 13-14) of \mbox{Algorithm \ref{alg2_new}}
\begin{align}
&\vect{x}_{k+1} = \vect{x}_k + \frac{1}{N}\sum\nolimits_{n=1}^N r_{n}^k \vect{v}_{k,n}\nonumber \\ \label{dx}
&\mathbf{d}_{\vect{x}_k} = \frac{1}{N}\sum\nolimits_{n=1}^N r_{n}^k \vect{v}_{k,n},
\end{align}
where $\mathbf{d}_{\vect{x}_k} = \vect{x}_{k+1} - \vect{x}_k$. Recall the variance formula
\begin{align}\label{var_formula}
    \text{Var}[\mathbf{d}_{\vect{x}_k}] = \EX[\mathbf{d}_{\vect{x}_k} \mathbf{d}_{\vect{x}_k}^\top] - \EX[\mathbf{d}_{\vect{x}_k}]\EX[\mathbf{d}_{\vect{x}_k}^\top].
\end{align}
We compute $\text{Var}[\mathbf{d}_{\vect{x}_k}]$ for the cases where $\vect{v}_{k,n}$ is drawn from either a normal or rademacher distribution. First, we begin with the case where $\vect{v}_{k,n}$ is sampled from a normal distribution. Recall, for a random vector $\vect{v}_{k,n}\sim \mathcal{N}(\mathbf{0}, \mathbf{I}_d)$, we have $\EX[\vect{v}_{k,n}] = \vect{0}$ and  $\EX[\vect{v}_{k,n} \vect{v}_{k,n}^\top] = \mathbf{I}_d. $ Compute $\EX[\mathbf{d}_{\vect{x}_k}]$ from \eqref{dx}
{\small
\begin{align*}
    \mathbb{E}[\mathbf{d}_{\vect{x}_k}] &= \mathbb{E} \left[\frac{1}{N} \sum\nolimits_{n=1}^N r_{n}^k \vect{v}_{k,n} \right] = \mathbb{E} \left[\frac{1}{N} \sum\nolimits_{n=1}^N (\delta_{n}^{k^\top} \vect{v}_{k,n}) \vect{v}_{k,n} \right] \\
    & = \mathbb{E} \left[\frac{1}{N} \sum\nolimits_{n=1}^N (\vect{v}_{k,n} \vect{v}_{k,n}^\top) \delta_{n}^{k} \right] = \frac{1}{N} \sum\nolimits_{n=1}^N  \delta_{n}^{k} = \bar{\delta}^k.
\end{align*}
}
Next, compute $\EX[\mathbf{d}_{\vect{x}_k} \mathbf{d}_{\vect{x}_k}^\top]$ in \eqref{var_formula}
\begin{align}\label{first_term}
    \mathbb{E}[\mathbf{d}_{\vect{x}_k} \mathbf{d}_{\vect{x}_k}^\top] = \frac{1}{N^2} \sum\nolimits_{n=1}^N \mathbb{E}[ ( (\delta_n^k)^\top \mathbf{v}_{k,n} )^2 \mathbf{v}_{k,n} \mathbf{v}_{k,n}^\top].
\end{align}
Note that $\vect{v}_k {\vect{v}_{k,n}}^\top
= \sum_{m=1}^{d} \sum_{p=1}^{d} v_{k,m} v_{k,n,p} \mathbf{e}_m \mathbf{e}_p^\top$,
and $(\delta_n^k{}^\top \vect{v}_{k,n})^2
= \left( \sum_{i=1}^d \delta_{n,i}^k v_{k,n,i} \right)^2
= \sum_{i=1}^d \sum_{j=1}^d \delta_{n,i}^k \delta_{n,j}^k v_{k,n,i} v_{k,n,j}$,
where $\mathbf{e}$ denotes the basis vector. Plugging back into \eqref{first_term}, we have 
{\small
\begin{align}
    &\EX[\mathbf{d}_{\vect{x}_k} \mathbf{d}_{\vect{x}_k}^\top] = \frac{1}{N^2}\sum\nolimits_{n=1}^N \EX[ \sum\nolimits_{i=1}^d \sum\nolimits_{j=1}^d \delta_{n,i}^k \delta_{n,j}^k v_{k,n,i} v_{k,n,j}  \nonumber \\
    &\sum\nolimits_{m=1}^{d} \sum\nolimits_{p=1}^{d} v_{k,n,m} v_{k,n,p} \mathbf{e}_m \mathbf{e}_p^\top]= \nonumber \\
    &  \frac{1}{N^2} \sum_{n=1}^N \!\EX[\sum_{i=1}^d \sum_{j=1}^d\! \sum_{m=1}^{d} \!\sum_{p=1}^{d} \!\delta_{n,i}^k \delta_{n,j}^k v_{k,n,i} v_{k,n,j}   v_{k,n,m} v_{k,n,p} \mathbf{e}_m \mathbf{e}_p^\top] \nonumber\\ \label{non-zero}
    &= \frac{1}{N^2} \!\!\sum_{n=1}^N \!\sum_{i=1}^d \!\sum_{j=1}^d\! \sum_{m=1}^{d} \!\sum_{p=1}^{d} \!\delta_{n,i}^k \!\delta_{n,j}^k\! \EX [v_{k,n,i} v_{k,n,j}  v_{k,n,m} v_{k,n,p}] \mathbf{e}_m \mathbf{e}_p^\top.
\end{align}
}
Note that, the last equality comes from the fact that $\vect{v}_{k,n}$ is independent from $\delta_n^k$. Now, there is different cases that $\EX [v_{k,n,i} v_{k,n,j}  v_{k,n,m} v_{k,n,p}]$ is non-zero in \eqref{non-zero}. \\
Case 1 ($i=j$ and $m=p$): In this case, \eqref{non-zero} simplifies to $\EX[\mathbf{d}_{\vect{x}_k} \mathbf{d}_{\vect{x}_k}^\top] = \frac{1}{N^2} \sum_{n=1}^N \sum_{i=1}^d  \sum_{m=1}^{d} (\delta_{n,i}^{k})^{2} \EX [v_{k,n,i}^2] \EX [v_{k,n,m}^2] \mathbf{e}_m \mathbf{e}_m^\top 
    = \frac{1}{N^2}\sum\nolimits_{n=1}^N \|\delta_n^k \|^2 \mathbf{I}_d$.
    
Case 2 ($i=m$ and $j=p$): In this case, \eqref{non-zero} simplifies to $
    \mathbb{E}[\mathbf{d}_{\vect{x}_k} \mathbf{d}_{\vect{x}_k}^\top] = \frac{1}{N^2} \sum_{n=1}^N \sum_{i=1}^d \sum_{j=1}^d \delta_{n,i}^k \delta_{n,j}^k \mathbb{E} [v_{k,n,i}^2] \mathbb{E}[v_{k,n,j}^2] \mathbf{e}_i \mathbf{e}_j^\top 
    = \frac{1}{N^2} \sum\nolimits_{n=1}^N \delta_{n}^k (\delta_{n}^k)^\top.$

Case 3 ($i=p$ and $m=j$): Similar to case 2, \eqref{non-zero} simplifies to
$
\mathbb{E}[\mathbf{d}_{\vect{x}_k} \mathbf{d}_{\vect{x}_k}^\top]
= \frac{1}{N^2} \sum\nolimits_{n=1}^N \delta_{n}^k (\delta_{n}^k)^\top.
$

Case 4 ($i=j=m=p$): In this case, \eqref{non-zero} simplifies to the following 
$\EX[\mathbf{d}_{\vect{x}_k} \mathbf{d}_{\vect{x}_k}^\top] =\frac{1}{N^2} \sum\nolimits_{n=1}^N \sum\nolimits_{i=1}^d  (\delta_{n,i}^k)^2 \EX [v_{k,n,i}^4] \mathbf{e}_i \mathbf{e}_i^\top= \frac{3}{N^2} \sum\nolimits_{n=1}^N  \|\delta_n^k\|^2 \mathbf{I}_d$. The fourth moment, i.e., $\EX[v_{k,n,i}^4]$ can be computed from \cite{isserlis1918formula}. Then, \eqref{non-zero} can be simplified as follows 
\begin{align*}
    \mathbb{E}[\mathbf{d}_{\vect{x}_k} \mathbf{d}_{\vect{x}_k}^\top] = \frac{2}{N^2} \sum\nolimits_{n=1}^N \delta_{n}^k (\delta_{n}^k)^\top + \frac{4}{N^2} \sum\nolimits_{n=1}^N \|\delta_n^k\|^2 \mathbf{I}_d.
\end{align*}

As a result, \eqref{var_formula} can be written as follows for the case where $\vect{v}_{k,n}$ drawn from a normal distribution with zero mean and unit variance.
\begin{align}\label{final_normal}
    \text{Var}_{\vect{v}_{k,n} \sim \mathcal{N}(\mathbf{0}, \mathbf{I}_d)}[\mathbf{d}_{\vect{x}_k}] = & \frac{2}{N^2} \sum\nolimits_{n=1}^N \delta_{n}^k (\delta_{n}^k)^\top \\
    & + \frac{4}{N^2} \sum\nolimits_{n=1}^N \|\delta_n^k\|^2 \mathbf{I}_d - \bar{\delta}^k (\bar{\delta}^k)^\top. \nonumber
\end{align}

Now, we compute $\text{Var}[\mathbf{d}_{\vect{x}_k}]$ for the case $\vect{v}_{k,n} \sim \text{Rademacher}^d$. It is important to note that for a random vector $\vect{v}_{k,n} \sim \text{Rademacher}^d$, we still have $ \EX[\vect{v}_{k,n}] = 0$ and  $\EX[\vect{v}_{k,n} \vect{v}_{k,n}^\top] = \mathbf{I}_d$. Similar to the proof for the normal distribution, cases $1, 2$, and $3$ are identical. However in the fourth case, where $i=j=m=p$, since the fourth moment of a Rademacher distribution is $1$, we get $\EX[\mathbf{d}_{\vect{x}_k} \mathbf{d}_{\vect{x}_k}^\top] = \frac{1}{N^2} \sum_{n=1}^N  \|\delta_n^k\|^2 \mathbf{I}_d$. Then, \eqref{non-zero} can be simplified as follows 
\begin{align*}
    \mathbb{E}[\mathbf{d}_{\vect{x}_k} \mathbf{d}_{\vect{x}_k}^\top] = \frac{2}{N^2} \sum\nolimits_{n=1}^N \delta_{n}^k (\delta_{n}^k)^\top + \frac{2}{N^2} \sum\nolimits_{n=1}^N \|\delta_n^k\|^2 \mathbf{I}_d.
\end{align*}

Thus, \eqref{var_formula} can be written as follows for the case where $\vect{v}_{k,n} \sim \text{Rademacher}^d$
\begin{align}\label{final_radem}
    \text{Var}_{\vect{v}_{k,n} \sim \text{Rademacher}^d}[\mathbf{d}_{\vect{x}_k}] = & \frac{2}{N^2} \sum\nolimits_{n=1}^N  \delta_{n}^k (\delta_{n}^k)^{\top}  \\
    & + \frac{2}{N^2} \sum\nolimits_{n=1}^N  \|\delta_n^k\|^2 \mathbf{I}_d - \bar{\delta}^k \bar{\delta}^{k \top}. \nonumber
\end{align}

From \eqref{final_normal} and \eqref{final_radem}, we have the following
\begin{align}
    \text{Var}_{\vect{v}_{k,n} \sim \mathcal{N}(\mathbf{0}, \mathbf{I}_d)}[\mathbf{d}_{\vect{x}_k}] & - \text{Var}_{\vect{v}_{k,n} \sim \text{Rademacher}^d}[\mathbf{d}_{\vect{x}_k}] \nonumber \\ 
    &= \frac{2}{N^2} \sum\nolimits_{n=1}^N  \|\delta_n^k\|^2 \mathbf{I}_d
\end{align}
which concludes the proof.
 \end{proof}




    


\medskip

\newcounter{mycounter}
\renewcommand{\themycounter}{A.\arabic{mycounter}}
\newtheorem{thmapp}[mycounter]{Theorem}

\end{document}